%% file: main.tex
\documentclass{article}

\usepackage[preprint]{neurips_2026}

\usepackage[utf8]{inputenc}
\usepackage[T1]{fontenc}
\usepackage{hyperref}
\hypersetup{
    colorlinks=true,
    linkcolor=red,
    citecolor=blue,
    urlcolor=green,
    linktocpage=true,
    breaklinks=true
}
\usepackage{url}
\usepackage{booktabs}
\usepackage{amsfonts}
\usepackage{nicefrac}
\usepackage{microtype}
\usepackage{xcolor}
\usepackage{algorithm}
\usepackage{algpseudocode}
\usepackage{float}
\usepackage{tikz}
\usetikzlibrary{arrows.meta}
\usepackage{enumitem}

\input{math_commands}

\usepackage[capitalize,noabbrev]{cleveref}

\title{PAC Learning with Bandit Feedback:\\ Sharp Sample Complexity in the Realizable Setting}
\author{
    Steve Hanneke\\
    Department of Computer Science\\
    Purdue University\\
    West Lafayette, IN, USA\\
    \texttt{steve.hanneke@gmail.com}\\
    \And
    Qinglin Meng\\
    Department of Computer Science\\
    Purdue University\\
    West Lafayette, IN, USA\\
    \texttt{meng160@purdue.edu}\\
    \And
    Shay Moran\\
    Faculties of Mathematics, Computer Science, and Data and Decision Sciences\\
    Technion – Israel Institute of Technology and Google Research\\
    Haifa, Israel\\
    \texttt{smoran@technion.ac.il}\\
    \And
    Amirreza Shaeiri\\
    Department of Computer Science\\
    Purdue University\\
    West Lafayette, IN, USA\\
    \texttt{amirreza.shaeiri@gmail.com}\\
}

\begin{document}

\maketitle

\begin{abstract}
We study the problem of multiclass PAC learning with bandit feedback in the realizable setting. In this framework, there is an unknown data distribution over an instance space $\mathcal{X}$ and a label space $\mathcal{Y}$, as in classical multiclass PAC learning, but the learner does not observe the labels of the i.i.d.\ training examples. Instead, in each round, it receives an unlabeled instance, predicts its label, and receives bandit feedback indicating only whether the prediction is correct. Despite this restriction, the goal remains the same as in classical PAC learning.

We provide a general characterization of the optimal sample complexity of this problem, sharp for every non-trivial concept class, up to logarithmic factors. Our characterization is based on a new combinatorial dimension, termed the \textit{bandit $\mathrm{DS}$ dimension}, defined via generalized combinatorial structures we call \emph{pseudo-boxes}. These extend the pseudo-cubes underlying the $\mathrm{DS}$ dimension by allowing a different number of neighbors in each coordinate. In contrast to the $\mathrm{DS}$ dimension, which governs the full-information setting by counting the number of coordinates in the pseudo-cube, the bandit $\mathrm{DS}$ dimension aggregates the number of neighbors across coordinates, leading to a characterization in which the sample complexity scales with the total number of neighbors.

We also propose a general learning algorithm achieving the upper bound, based on an algorithmic principle called \emph{ListCascade}, which connects bandit learning to list learning and may be of independent interest.

\end{abstract}

\newpage

\input{Main/Introduction}
\input{Main/RelatedWork}

\input{Main/Preliminaries}

\input{Main/Results}

\input{Main/Conclusion}

\input{Main/Acknowledgments}

\clearpage

\bibliographystyle{plainnat}
\bibliography{References}
\clearpage

\appendix
\input{Appendix/Definitions}

\input{Appendix/Deferred}
\input{Appendix/effective}

\input{Appendix/Lemmata}

\clearpage


\end{document}

%% file: math_commands.tex
\usepackage{amsmath}
\usepackage{amssymb}
\usepackage{mathtools}
\usepackage{amsthm}
\usepackage{thmtools}
\usepackage{etoolbox}

\usepackage{amsfonts}
\usepackage{mathrsfs}
\usepackage{dsfont}
\usepackage{bm}

\usepackage{relsize}
\usepackage{nicefrac}
\usepackage{centernot}

\theoremstyle{plain}
\newtheorem{theorem}{Theorem}[section]
\newtheorem{lemma}[theorem]{Lemma}

\theoremstyle{definition}
\newtheorem{definition}[theorem]{Definition}

\theoremstyle{remark}
\newtheorem{remark}[theorem]{Remark}

\AtBeginEnvironment{proof}{\setcounter{localclaim}{0}}



\DeclarePairedDelimiter{\ceil}{\lceil}{\rceil}

\usepackage{tcolorbox}
\tcbset{
    rounded corners,
    colback = white,
    before skip = 0.25cm,
    after skip = 0.5cm,
    boxrule = 2pt,
    arc = 10pt
}

\newcommand{\avd}{\mathsf{avd}^L}

\newcommand{\outdeg}{\mathsf{outdeg}^L}
\newcommand{\degree}{\mathsf{deg}^L}

\newcommand{\iid}{\mathrm{i.i.d.}}
\newcommand{\err}{\cL_{\cD}}
\newcommand{\maj}{\mathrm{Maj}}
\newcommand{\errp}{\err^{\prime}}

\newcommand{\sigmal}{\sigma^L}

\newcommand{\EE}{\mathbb{E}}

\newcommand{\cZp}{\mathcal{Z}^{\prime}}
\newcommand{\dds}{d^ L _{\operatorname{DS}}}
\newcommand{\ole}{\mathrm{L}\text{-}\operatorname{E}}

\newcommand{\dexp}{d^ L _{\operatorname{E}}}
\newcommand{\ddsh}{d^{\lceil L/2\rceil}_{{\operatorname{DS}}}}

\newcommand{\cA}{\mathcal A}

\newcommand{\cD}{\mathcal D}

\newcommand{\cF}{\mathcal F}

\newcommand{\cH}{\mathcal H}

\newcommand{\cL}{\mathcal L}

\newcommand{\cX}{\mathcal X}
\newcommand{\cY}{\mathcal Y}
\newcommand{\cZ}{\mathcal Z}

%% file: Main/Introduction.tex
\section{Introduction} \label{Introduction}

A natural variant of multiclass PAC learning arises when the learner does not observe labels, but only receives binary feedback indicating whether its predictions are correct. This setting, known as multiclass PAC learning with bandit feedback, was introduced and studied in the PAC framework by \cite*{daniely2011multiclass, daniely2015multiclass}. Formally, there is an unknown distribution $\mathcal{D}$ over an instance space $\mathcal{X}$ and a label space $\mathcal{Y}$.
The training process proceeds over~$n$ rounds: in each round, a sample $(x, y) \sim \mathcal{D}$ is drawn, the instance $x$ is revealed, and the learner predicts a label $\hat{y} \in \mathcal{Y}$. Instead of observing the true label, the learner only receives bandit feedback indicating whether $\hat{y} = y$. The goal is to output a hypothesis $\hat{h} : \mathcal{X} \to \mathcal{Y}$ that generalizes well to $\mathcal{D}$. We assume a hypothesis class $\mathcal{H} \subseteq \mathcal{Y}^{\mathcal{X}}$ and focus on the realizable setting, where $\mathcal{D}$ is consistent with some $h \in \mathcal{H}$. 

There are at least two primary motivations for studying multiclass learning under bandit feedback. First, bandit feedback is often substantially cheaper to obtain than full information. For instance, answering the question ``Is this an image of a horse?'' is significantly faster than selecting the correct label from a long list of possible categories. Moreover, in many real-world scenarios, this restricted form of feedback is not merely a cost-saving choice but an inherent limitation of the environment. As an illustration, imagine a pandemic caused by an unknown virus, where we aim to develop an AI model capable of accurately recommending the appropriate medication for patients drawn i.i.d.\ from the same underlying population. Suppose, for simplicity, that each patient responds positively to exactly one drug from a fixed set of available drugs. Due to various challenges associated with human clinical trials, including ethical and financial constraints, only a limited number of patients can be enrolled. Naturally, the only feedback we can obtain from each trial is binary—whether the drug tested was effective or not for that patient. We refer the reader to \cite{woodroofe1979one, li2010contextual, agarwal2016making} for additional practical examples.

In this work, we study the following fundamental and naturally arising question:

\begin{center}
\textit{Given a concept class $\mathcal{H} \subseteq \mathcal{Y}^{\mathcal{X}}$, how many training samples are necessary and sufficient for PAC learning $\mathcal{H}$ when the learner receives only bandit feedback?}
\end{center}

The work of \cite{daniely2011multiclass, daniely2015multiclass} provided an initial result on the above question, among others. However, their upper and lower bounds have a multiplicative gap of $K$ up to logarithmic factors, where $K$ is the size of the effective label space. In particular, at constant confidence and for a concept
class $\mathcal H$ with $\mathrm{ND}(\mathcal H)\ge2$,
they showed a lower bound of $\widetilde{\Omega} \left ( \operatorname{ND}(\mathcal{H})/\epsilon\right )$ and an upper bound of $\widetilde{O} \left ( K \operatorname{ND}(\mathcal{H})/ \epsilon \right )$ on the optimal sample complexity of multiclass PAC learning with bandit feedback, where $\operatorname{ND}(\mathcal{H})$ denotes the Natarajan dimension of $\mathcal{H}$, the standard multiclass analogue of the VC dimension, defined by the largest set of points that can be shattered using two candidate labels per point.

\subsection{Overview of the Main Results}

Our main contribution is a new combinatorial complexity parameter that sharply characterizes the sample complexity of PAC learning with bandit feedback. The definition is based on a natural generalization of classical combinatorial structures that underlie PAC learning theory.

We begin with a simple object. An $N_1 \times N_2 \times \cdots \times N_d$ \emph{box} is a cartesian product $A_1 \times A_2 \times \cdots \times A_d$, where each $A_i \subseteq \mathcal{Y}$ has size $N_i$. Boxes capture independent variation across coordinates and play a central role in classical dimensions. For example, a Natarajan cube corresponds to a $2 \times 2 \times \cdots \times 2$ box, while its natural generalization to list learning corresponds to $(\ell+1) \times \cdots \times (\ell+1)$ boxes.

In our setting, however, we require a more flexible notion. We say that a finite non-empty set $T \subseteq \mathcal{Y}^d$ is an $N_1 \times \cdots \times N_d$ \emph{pseudo-box} if, every vector $f \in T$ and every $i\in[d]$, there are at least $N_i$
distinct vectors in $T$ that differ from $f$ at coordinate
$i$ and agree with $f$ at every other coordinate. Pseudo-boxes relax the rigid Cartesian product structure while preserving its local combinatorial richness. In particular, pseudo-cubes (the case $N_i=1$) recover the classical structure underlying the $\operatorname{DS}$ dimension, while larger values of $N_i$ correspond to its extensions used in list learning~\cite{daniely2014optimal, brukhim2022characterization,charikar2023characterization}.

The key point is that such structures precisely capture the difficulty of learning. In the full-information setting, complexity is governed by the number of coordinates in a pseudo-box, leading to dimensions such as the $\operatorname{DS}$ dimension. In contrast, under bandit feedback, the difficulty accumulates across coordinates: each direction contributes separately to the complexity.

Motivated by this perspective, we define the \emph{bandit $\operatorname{DS}$ dimension} of a concept class $\mathcal{H}$, denoted $\operatorname{BDS}(\mathcal{H})$, as the maximum total size $\sum_{i=1}^d N_i$ of a pseudo-box realizable by $\mathcal{H}$. We refer the reader to \autoref{Combinatorial Complexity Parameters} for the formal definition and further discussion.

Equipped with this dimension, we establish matching upper and lower bounds on the optimal sample complexity of multiclass PAC learning with bandit feedback, resolving an open problem posed by \cite{daniely2011multiclass, daniely2015multiclass} up to logarithmic factors.

\begin{theorem} \label{intro: main-theorem}
Let $\mathcal{H} \subseteq \mathcal{Y}^{\mathcal{X}}$ be a non-trivial concept class with $\operatorname{BDS}(\mathcal{H})<\infty$, and denote by $\mathbf{m}^{B}_{\mathcal{H}}$ the optimal sample complexity of multiclass PAC learning with bandit feedback (see \cref{def: bandit-pac-learnability}). Then, for every $0 < \epsilon \leq 1/16$ and $0 < \delta \leq 1/16$,
\begin{equation*}
    \mathbf{m}^{B}_{\mathcal{H}}(\epsilon, \delta) \in \widetilde{\Theta} \left ( \frac{\operatorname{BDS}(\mathcal{H})}{\epsilon} \right ),
\end{equation*}
where $\widetilde{\Theta}$ hides polylogarithmic factors in $K$, $1/\epsilon$, and $1/\delta$.
\end{theorem}

The upper bound in \cref{intro: main-theorem} is based on a new algorithmic framework, which we call \textit{ListCascade}. At a high level, ListCascade uses a sequence of list learners with progressively shrinking list sizes, thereby converting list learning guarantees into a bandit-feedback learner whose complexity is governed by the bandit $\operatorname{DS}$ dimension. 
A key ingredient in the analysis is a refined sample complexity bound for list PAC learning, which may also be of independent interest; see \cref{Techniques} for more details. It remains interesting to understand to what extent the polylogarithmic overhead in this general upper bound can be reduced.

We also place our results in the broader landscape of bandit learning. 
Although multiclass learning with bandit feedback can be viewed as a special case of contextual bandits, a direct reduction to this general setting misses a key structural feature of classification: the induced reward function is sparse. 
Recent works have exploited this sparsity to obtain sharper sample-complexity guarantees~\cite{erez2024real, erez2024fast, erez2025list}, primarily for finite concept classes and in the agnostic regime.
Our results show that, in the realizable PAC setting, this additional structure enables a combinatorial characterization of the near-optimal sample complexity, in contrast to general bandit learning problems where such characterizations are provably impossible~\cite{pmlr-v195-hanneke23d, brukhim2025hardnessbanditlearning}.

\subsection{Overview of the Techniques } \label{Techniques}

\textbf{List learning.}
Our first ingredient is list learning, which serves as the
intermediate object connecting bandit feedback and multiclass prediction. In standard multiclass
learning, the learner predicts a single label for each instance. In list learning, the learner is
allowed to output a short list of candidate labels, and the prediction is considered correct as long
as the true label is contained in the list. This relaxation is particularly useful under bandit
feedback: even if the learner cannot immediately identify the correct label, it can progressively
shrink the set of plausible labels while preserving the true one with high probability.

\textbf{ListCascade.}
Our main algorithmic idea is a cascading procedure that repeatedly learns smaller lists from
bandit feedback. Here, for simplicity, we assume that the size of the label space and effective label space are both $K$. This should be contrasted with the classical reduction of \citet{daniely2011multiclass}, which explores by predicting a uniformly random label from
\(\mathcal{Y}\). Since the correct label is revealed only with probability \(1/K\), one needs about
\(K\) bandit examples to obtain one full-information example, leading to the
\(\widetilde{O}(K \operatorname{ND}(\mathcal{H})/\epsilon)\) upper bound.

ListCascade replaces this one-shot exploration by a sequence of list-learning stages. Initially,
the learner considers all \(K\) labels as possible. In the first epoch, it predicts uniformly from
this full list and uses the revealed examples to learn a list predictor of size roughly \(K/2\).
In the next epoch, the learner only explores within this shorter list, and then learns an even
shorter list. Repeating this process, the list size is halved from \(K\) to \(K/2\), then to
\(K/4\), and so on, until the final list has size one; this final list predictor is the returned
multiclass hypothesis.

The key point is that the cost of each epoch scales with the current list size, rather than with
the original number of labels \(K\). In epoch $t$, the exploration list has size at most $2L_{t-1}$,
and the majority vote of the $L_t$-list predictors has size at most
$2L_t-1$. Writing $r_t=\lceil L_t/2\rceil$, the list-learning
guarantee below gives an epoch cost of
$\widetilde O\bigl(
L_{t-1}\cdot (r_t\text{-DS}(\mathcal H))/\epsilon
\bigr)$.
Since $L_{t-1}\le 4r_t$ and
$r_t\cdot(r_t\text{-DS}(\mathcal H))
\le \operatorname{BDS}(\mathcal H)$,
summing over the logarithmically many epochs gives the desired
upper bound, up to logarithmic factors.

\textbf{Sample complexity of list learning.}
It remains to establish the list-learning guarantee required in each epoch of
\textsc{ListCascade}. We use the one-inclusion list algorithm studied by
\citet{charikar2023characterization}. By the analysis of
\citet{hanneke2026optimalsauerlemmakary}, its sample complexity is controlled by the corresponding
list exponential dimension: learning an \(L\)-list predictor requires
\(\widetilde{O}(\operatorname{L\text{-}E}(\mathcal{H})/\epsilon)\) samples, where $\operatorname{L\text{-}E}(\mathcal{H})$ denotes the $L$-exponential dimension. Informally,
this dimension is the list-learning analogue of the exponential dimension: it captures the largest
set of points on which the class realizes exponentially many label patterns at scale $L$.
Prior comparisons between list exponential dimension and list DS dimension yield a bound of order
\(\widetilde{O}(L \cdot \operatorname{L\text{-}DS}(\mathcal{H})/\epsilon)\).
While sufficient for characterizing list learning, this extra factor \(L\) is too costly for our
bandit application, since it would multiply the exploration cost in \textsc{ListCascade} and lead to a
suboptimal bound.

Our contribution is to sharpen this dimension comparison in the regime needed by the cascade. Using
the optimal Sauer lemma of \citet{hanneke2026optimalsauerlemmakary}, we show that the
\(L\)-exponential dimension is controlled, up to logarithmic factors in \(K\), by the
\(\mathrm{\lceil L/2\rceil}\)-DS dimension. Hence the suffix-majority aggregation of the one-inclusion
$L$-list predictors learns a list predictor of size
$\min\{2L-1,K\}$ with sample complexity
$\widetilde O\bigl(
(\lceil L/2\rceil\text{-DS}(\mathcal H))/\epsilon
\bigr)$. This removes the extraneous list-size factor and is the key ingredient that allows
\textsc{ListCascade} to achieve the near-optimal bandit sample complexity.

\textbf{Lower Bound.}
The lower bound is obtained by turning the pseudo-box witnessing
$\operatorname{BDS}(\mathcal H)$ into a hard bandit instance for non-trivial classes. Suppose
$x_1, x_2, \ldots, x_d$ are BDS-shattered with multiplicities
$N_1, N_2, \ldots, N_d$, where $\sum_i N_i=\operatorname{BDS}(\mathcal H)$ and $d \geq 2$; after reordering, assume $N_1=\min_i N_i$, so that $\sum_{i=2}^d N_i\ge \operatorname{BDS}(\mathcal H)/2$. We place most of the distributional mass on one anchor point and distribute the remaining
$\Theta(\epsilon)$ mass over the other points in proportion to $N_i$. The target restriction is then chosen uniformly from the finite pseudo-box witness, after fixing one representative hypothesis for each distinct restriction. For each point $x_i$, the multiplicity $N_i$ represents
the number of locally indistinguishable alternatives that differ only at
$x_i$. Conditioned on all the other coordinates of the target, the label at $x_i$ is uniform among at least $N_i+1$ possible labels, while each visit to $x_i$ only tells the learner whether one guessed label is correct; hence, after $c$ such visits and one final prediction, the learner can correctly identify at most $c+1$ of these possibilities. With fewer than $\Theta(\sum_{i=2}^d N_i/\epsilon)=\Theta(\operatorname{BDS}(\mathcal H)/\epsilon)$ samples, this yields error at least $\epsilon$ with constant probability for some target concept.  Therefore any bandit PAC learner requires
$\Omega(\operatorname{BDS}(\mathcal H)/\epsilon)$ samples, matching our upper
bound up to logarithmic factors.

%% file: Main/RelatedWork.tex
\section{Related Work} \label{Related Work}

\paragraph{Contextual Bandit.} The bandit model occupies a central place in decision theory. Its origins trace back to the pioneering work of \cite{thompson1933likelihood}, and the field gained prominence following Herbert Robbins’ short but influential contribution \cite{robbins1952some}, which introduced the classical multi-armed bandit problem. Subsequently, Michael Woodroofe wrote the first paper on the fundamental problem of contextual bandits \cite{woodroofe1979one}. Moreover, the literature on this topic has expanded dramatically since the landmark contributions of \cite{auer2002finite, auer2002nonstochastic}. Beginning with works such as \cite{dudik2011efficient, pmlr-v32-agarwalb14}, a particularly influential technical idea has been the use of reduction-based approaches, which relate contextual bandits to supervised learning primitives, through importance weighting. For a comprehensive treatment of this subject, we refer the reader to \cite{bubeck2012regret, MAL-068, lattimore2020bandit, foster2021statistical}.

Additionally, there is a recent paper on the problem studied in the current work within the framework of universal learning \cite{hanneke2025universal}. Nevertheless, in this framework, the dependence on the number of labels appears to be fundamentally different from that in the PAC learning framework. We refer the reader to \cite{hanneke2024universal, hanneke2025universal, hanneke2025nonuniform} for more discussion of this phenomenon.

\paragraph{List Learning.} In recent years, there has been a growing body of work on list learning across both the theoretical and empirical branches of machine learning. On the theoretical side, list learners play a key role in the breakthrough result of \cite{brukhim2022characterization}, which provides a combinatorial characterization of multiclass PAC learnability. This topic has since developed in several directions, including PAC learning \cite{charikar2023characterization, hanneke2024improved, hannekerepresentation, cohen2026natarajan}, online learning \cite{moran2023list}, boosting \cite{pmlr-v195-brukhim23a, brukhim2024multiclass}, sample compression schemes and uniform convergence \cite{hanneke2024list}, real-valued regression \cite{pabbaraju2024characterizationlistregression}, transductive online learning \cite{hanneke2025a}, differentially private PAC learning \cite{hanneke2025private}, combinatorial semi-bandits \cite{erez2025list}, language generation \cite{charikar2025characterizationlistlanguageidentification}, and SSP lemmas \cite{hanneke2026optimalsauerlemmakary}. Moreover, list learning is also naturally connected to the fundamental framework of list-decodable learning, which is inspired by list decoding in coding theory. This framework was introduced by \cite{balcan2008discriminative} in the context of clustering (see also \cite{10.5555/1813231.1813269}) and subsequently developed by \cite{charikar2017learning} in the context of robust statistics. For a detailed exposition on list-decodable learning, for example, refer to Chapter~5 of the recent textbook by \cite{diakonikolas2023algorithmic}.


%% file: Main/Preliminaries.tex
\section{Preliminaries} \label{Notations, Definitions, and Preliminaries}


We consider a learning setting in which a learner is tasked with classifying instances from an instance space $\cX$ with a single label from a set of $K$ possible labels $\cY = \{1,\cdots,K\}$. A stochastic multiclass classification instance is specified by a hypothesis class $\cH\subseteq\cY^{\cX}$ and a joint distribution over example-label pairs over $\cX \times \cY$.

We make the standard measurability assumptions that the hypotheses
are measurable and that the fixed orientations, list-padding rules,
and pointwise relabelings used below can be chosen so that the
resulting prediction rules are jointly measurable in the sample,
test instance, and internal randomness.

We say that a concept class $\cH\subseteq\cY^\cX$ is \emph{non-trivial} if
$0<\operatorname{BDS}(\cH)<\infty$ and there exist $d\ge2$, a sequence
$S\in\cX^d$, a vector $N\in\mathbb{N}^d$, and a finite nonempty
$\cF\subseteq\cH$ such that $(S,N)$ is BDS-shattered by $\cH$ with witness
$\cF$ and
\[
\sum_{i=1}^d N_i=\operatorname{BDS}(\cH).
\]
In other words, the bandit DS dimension of $\cH$ is attained by a BDS
witness supported on at least two instances.

Before describing the bandit feedback model, we introduce the necessary notions used to state
the PAC guarantee. The performance of a hypothesis is measured by its expected
classification error.
\begin{definition}[Error rate]
    Let $\cD$ be a distribution over $\cX\times\cY$ and let $h\in\cY^{\cX}$. We define the error rate of $h$ with respect to $\cD$ as follows:
    \[\cL_{\cD}(h) = \Pr_{(x, y) \sim \mathcal{D}} \left[ h(x) \neq y \right].\]
\end{definition}
We focus on the realizable setting, where the hypothesis class contains a perfect
classifier for the data distribution. Formally, we have the following definition. 
\begin{definition}[Realizable Distribution] \label{def: realizable distribution}
We say that a data distribution $\mathcal{D}$ is realizable by a concept class $\mathcal{H}$ if there exists $h^{\star} \in \mathcal{H}$ such that $\mathcal{L}_{\mathcal{D}}(h^{\star}) = 0$.
\end{definition}

We now specify the bandit feedback learning protocol. In the bandit feedback setup, the learner does not observe the true label $y_i$ directly; instead, it only
observes whether its prediction is correct. Learning proceeds through the following protocol, over $i = 1,2,\cdots,m$:
\begin{enumerate}
    \item The environment samples an instance $(x_i, y_i)$ from $\mathcal{D}$ and reveals $x_i$ to the learner.
    \item The learner predicts a label $\hat{y}_i$ from $\mathcal{Y}$.
    \item Nature reveals the feedback $f_i = \mathds{1} \{ y_i = \hat{y}_i \}$ to the learner.
\end{enumerate}
Finally, based on $\{ (x_1, \hat{y}_1, f_1),\ldots, (x_m, \hat{y}_m, f_m) \}$, the learner returns a hypothesis $\hat{h}:\cX\rightarrow\cY$.


The goal of PAC learning is to output, with high probability, a hypothesis whose
expected error is small and this requirement
is formalized by the following sample complexity notion.
\begin{definition}[Sample complexity for realizable multiclass PAC learning with bandit feedback]
\label{def: bandit-pac-learnability}
For $\epsilon,\delta\in(0,1)$, the realizable PAC sample complexity of $\cH$ with
bandit feedback, denoted by $\mathbf{m}_{\cH}^{B}(\epsilon,\delta)$, is the smallest
integer $m$ such that there exists a learning algorithm satisfying the following:
for every distribution $\cD$ realizable by $\cH$, after $m$ rounds of interaction
with bandit feedback, the returned hypothesis $\hat h$ satisfies
\[
    \Pr
    \left[
        \cL_{\cD}\left(
        \hat{h}
        \right) > \epsilon
    \right]
    \leq \delta,
\]
where the probability is over the sample $S\sim\cD^m$ and any internal randomness of
the algorithm.
\end{definition}


Next, we introduce list learning \citep{charikar2023characterization}, which will be used as an intermediate object in our
analysis. In list learning, instead of predicting a single label, the learner outputs a
list of candidate labels for each instance. A list predictor of size $L$ is a mapping
$\mu:\cX\to \{Y\subseteq\cY: |Y|\le L\}$. Its error rate is defined as $\cL_{\cD}(\mu) = \Pr_{(\cX,\cY)\sim\cD}[y \not\in \mu(x)]$.

Next, we define the majority vote over lists, which we will use later to convert list predictors with expected-error guarantees into a list predictor satisfying a PAC guarantee.

\begin{definition}[Majority vote for lists \citep{hanneke2024improved}]
    For $n$ lists $\mu_1, \dots, \mu_n$ each of size $L$, we define their majority vote to be $\mu = \maj(\mu_1, \dots, \mu_n)$ with
\begin{equation*}
    \mu(x) = \maj(\mu_1, \dots, \mu_n)(x) := \{y \in \cY:|i\in[n]: y\in\mu_i(x)| > n/2\}, \forall x\in \cX.
\end{equation*}
\end{definition}
Note that the majority-vote list $\mu$ has size at most $2L-1$, rather than
$L$. For notational convenience, we pad $\mu(x)$ with arbitrary labels
from $\cY \setminus \mu(x)$ so that all output lists have size exactly
$\min\{2L-1,K\}$. When $L=1$, this is strict-majority voting, with an arbitrary
fallback label when no label receives a strict majority.

%% file: Main/Results.tex
\section{Algorithm and Analysis} \label{Results}
In this section, we present and analyze one of our main contributions: a near optimal bandit multiclass classification algorithm, detailed in \cref{algo: Bandit PAC Multiclass Classification}, which will be shown to satisfy the following realizable PAC guarantee. 

\begin{remark}
For simplicity of presentation, throughout this section we state the bounds in terms of the label space size $K=|\mathcal{Y}|$. The same arguments can be formulated in terms of the effective label space size of the class, which can be substantially smaller than $K$. We discuss this refinement in the Appendix~\ref{sec:effective-label-space}.
\end{remark}

\begin{theorem}[PAC Sample Complexity Upper Bound for Multiclass Classification with Bandit-feedback]\label{theorem: upper bound for bandit feedback}
    Let $\cH \subseteq \cY^\cX$ be a concept class with bandit DS dimension $\operatorname{BDS}(\cH) = d_{DS}^{B} <\infty$. If $K=1$, then $\mathrm{\mathbf m}_{\cH}^B(\epsilon,\delta)=0$. For $K\ge2$, we have
    \begin{equation*}
       \mathrm{\mathbf{m}}_{\cH}^{B}(\epsilon,\delta)=  O\left(\frac{d_{DS}^{B} (\log K)^3 + K\log K\log\left(\frac{2 \log K}{\delta}\right)}{\epsilon} \right).
    \end{equation*}
\end{theorem}

For comparison, we also prove the following matching lower bound up to logarithmic
factors.
\begin{theorem}[PAC Sample Complexity Lower Bound for Multiclass Classification with Bandit-feedback]\label{theorem: lower bound for bandit feedback}
    Let $\cH \subseteq \cY^\cX$ be a non-trivial concept class with bandit DS dimension $\operatorname{BDS}(\cH) = d_{DS}^{B} <\infty$. Let $K = |\cY| \ge 2$. Then, for every $0 < \epsilon \leq 1/16$ and $0 < \delta \leq 1/16$, we have 
    \begin{equation*}
    \mathrm{\mathbf{m}}_{\cH}^{B}(\epsilon,\delta) = \Omega\left(\frac{d_{DS}^{B} + \log\left(\frac{1}{\delta}\right)}{\epsilon} \right).
    \end{equation*}
\end{theorem}

\begin{algorithm}[t]
\caption{\textsc{ListCascade}: Bandit PAC Multiclass Classification via List Learners} 
\label{algo: Bandit PAC Multiclass Classification}
\begin{flushleft}
{\bf Input:} An unlabeled sample $S=(x_1,x_2,\ldots,x_m)$.
{\bf Output:} A multiclass classification hypothesis $h_S$. \\
\end{flushleft}
\begin{algorithmic}[1]
\State Set $m_0=0$, $\mu_0(x) = \cY$ for any $x\in\cX$, and $L_0=\lceil K/2\rceil$.
\For{epoch $t = 1, 2,\cdots,\left\lfloor\log K\right\rfloor$}
\State Initialize $S_t=()$.
    \For{$i = m_ {t-1} + 1,m_{t-1} + 2,\cdots, m_t$}
        \State Observe $x_i$, predict $\hat{y_i} \sim \mathrm{Unif}\left(\mu_{t-1}(x_i)\right)$ and receive feedback $\mathds{1}\{y_i = \hat{y_i}\}$.
        \State Update $S_t\leftarrow (S_t,(x_i,y_i)).$ if $y_i = \hat{y_i}$, otherwise $S_{t}$ is unchanged.
    \EndFor
    \State Let $L_t=\ceil{K/2^{t+1}}$ and construct $\mu_t$ from $S_t$ according to \cref{eq: mu_t calculation}.
\EndFor
\State Set $h_S(x)$ to be the unique element of
$\mu_{\left\lfloor\log K\right\rfloor}(x)$ for all $x\in\cX$.
\end{algorithmic}
\end{algorithm}
\cref{algo: Bandit PAC Multiclass Classification} proceeds by gradually
shrinking the list of candidate labels. In epoch $t$, the learner explores using
the previous list predictor $\mu_{t-1}$ by predicting uniformly from
$\mu_{t-1}(x)$, and keeps an example only when the bandit feedback is positive. The retained examples form the labeled sample $S_t$.  If $n_t<4$, let $\mu_t$ be any fixed list predictor of the required size;
otherwise define $\mu_t$ by \cref{eq: mu_t calculation}:
\begin{equation}\label{eq: mu_t calculation}
    \mu_t
    =
    \maj\left(
    \mu_{S_{t,\le \lceil n_t/4\rceil}}^{L_t},
    \dots,
    \mu_{S_{t,\le n_t-1}}^{L_t}
    \right),
\end{equation}
where $n_t=|S_t|$, and $\mu_{S_{t,\le n}}^{L_t}$ is the output of the
one-inclusion list learner of
\citet{charikar2023characterization,hanneke2026optimalsauerlemmakary}
trained on the first $n$ examples of $S_t$ with list size $L_t$. The majority
vote increases the list size to $2L_t-1$, so with
$L_t=\ceil{K/2^{t+1}}$, the effective list size remains on the order of
$K/2^t$. Thus the algorithm reduces the candidate label set by a constant factor
in each epoch. For completeness, we include the one-inclusion list learner in
\cref{algo: learn a L-list} in the appendix. The precise epoch lengths $m_t-m_{t-1}$ are specified in the proof of
\cref{claim: induc}.

\subsection{Proof of \cref{theorem: upper bound for bandit feedback}}
In this section, we will provide the analysis of \cref{algo: Bandit PAC Multiclass Classification} and the proof of \cref{theorem: upper bound for bandit feedback}. 
We start with a technique lemma that upper bounds the $L$-exponential dimension by $(L/2)$-DS dimension.
\begin{lemma}[Controlling the $L$-exponential dimension by $(L/2)$-DS dimension]\label{lemma: bound exp by L/2 DS}For every $n \in \mathbb{N}$, $K \ge 2$, $L\ge1 $, and for every concept class $\cH\subseteq[K]^{n}$ with $\mathrm{\lceil L /2 \rceil}\text{-}\operatorname{DS}(\cH) = \ddsh$ and $\ole(\cH) = \dexp< \infty$, we have
$\dexp \le  6\ddsh\log K .$
\end{lemma}
\begin{proof}[Proof of \cref{lemma: bound exp by L/2 DS}]
    If $d_{DS}^{\lceil L/2\rceil}=0$ or $d_E^L=0$, then the claim is immediate. If $d_E^L\le d_{DS}^{\lceil L/2\rceil}$, then the claim is also immediate. Hence, assume
    \[
    d_E^L>d_{DS}^{\lceil L/2\rceil}\ge1.
    \]
    Consider the set $S$ which realizes $\dexp$, so that $|S| = \dexp$. By the definition of $\dexp$, we have
    \[|\cH|_{S}|\ge(L+1)^{\dexp}.\]
    On the other hand, by Theorem 2.1 of \cite{hanneke2026optimalsauerlemmakary},
    \[|\cH|_{S}|\le \left\lceil\frac{L}{2}\right\rceil^{\dexp - \ddsh}\left(\frac{eK\dexp}{\ddsh}\right)^{\ddsh}.\] 
    Combining the preceding inequalities and proceeding as in \cref{complete proof of dexp} gives
    \[
    d_E^L\le 6d_{DS}^{\lceil L/2\rceil}\log K.
    \]
     The full calculation is deferred to \cref{complete proof of dexp}.
\end{proof}
Based on \cref{lemma: bound exp by L/2 DS} as well as Lemma 3.20 and Lemma 3.22 of \cite{hanneke2026optimalsauerlemmakary}, we can upper bound the leave-one-out error of \cref{algo: learn a L-list} as follow.

\begin{lemma}[Leave-one-out error bound for list learning]
    \label{lemma: loo error bound}
    For every concept class $\cH \subseteq \cY ^{\cX}$ with $\mathrm{\lceil L / 2\rceil}\text{-}\operatorname{DS}(\cH) = \ddsh$,
    let $\cD$ be an $\cH$-realizable distribution and let $n > 0$ be an integer.

    Moreover, for every realizable sample
    $S=((x_i,y_i))_{i=1}^n$,
    \[
    \sum_{i=1}^n
    \mathds{1}\!\left\{
    y_i\notin\mu^L_{S^{-i}}(x_i)
    \right\}
    \le
    6d_{DS}^{\lceil L/2\rceil}\log K,
    \]
    where $S^{-i}$ denotes $S$ with its $i$th example removed. As a result, the prediction error can be bounded as
    \begin{equation*}
        \Pr_{(S,(x,y)) \sim \cD^{n+1}}\left[y \not \in \mu^{L}_S(x)  \right] \leq \frac{6\ddsh\log K}{n+1},
    \end{equation*}
    where $\mu^{L}_S = \cA_{\cH}^{L}(S)$ is the output of \cref{algo: learn a L-list}.
\end{lemma}
We now state and prove our list learning PAC bound, building on \Cref{lemma: loo error bound} and Theorem 2.1 in \cite{aden2023optimal}. 
\begin{lemma}\label{lemma: loo to pac}
Fix a concept class $\cH \subseteq \cY^{\cX}$ with $\mathrm{\lceil L / 2\rceil}\text{-}\operatorname{DS}(\cH) = \ddsh$. There exists a learning rule, depending only on $\mathcal H$
and $L$, that maps each realizable sample $S$ to a list predictor
$\mu=\mu(S)$ of size $\min\{2L-1,K\}$, such that for every
$\mathcal H$-realizable distribution $D$, every integer $n\ge 1$,
every $\delta\in(0,1)$, and $S\sim D^n$, the output satisfies
\[
\cL_{\cD}\left(\mu\right)\le 14.46\left(\frac{6\ddsh\log K}{n}+\frac{1}{n}\log\left(\frac{2}{\delta}\right)\right),
\]
with probability at least $1-\delta$ over the randomness of $S$.
\end{lemma}
\begin{proof}[Proof of \Cref{lemma: loo to pac}]
For $n<4$, the claimed bound is trivial, so assume $n\ge4$. Given the full training sample $S$, let $(\mu_{S_{\le t}}^{L})_{t = \lceil n / 4\rceil}^{n-1}$ be the one-inclusion hypergraph predictors trained on the sequence of growing training samples $(S_{\le t})_{t=\lceil n / 4\rceil}^{n-1}$ output by \cref{algo: learn a L-list}. Because these predictors are symmetric in their training sample and they have a leave-one-out error guarantee (\Cref{lemma: loo error bound}), we can apply Theorem 2.1 in \cite{aden2023optimal} to the suffix of the sequence to get that 
\[
\frac{4}{3n}\sum_{t=\lceil n/4\rceil}^{n-1}\cL_{\cD}\left(\mu_{S_{\le t}}^{L}\right) \le 4.82\left(\frac{6\ddsh\log K}{n}+\frac{1}{n}\log\left(\frac{2}{\delta}\right)\right),
\]
with probability at least $1-\delta$ over the randomness of $S$. Their proof applies with $\lceil n/4\rceil$ in place of $n/4$; the only
change is that
\[
\sum_{i=\lceil n/4\rceil+1}^{n}\frac1i
\le
\ln\frac{n}{\lceil n/4\rceil}
\le\ln4.
\]
Whenever this bound holds, the majority vote of the one-inclusion hypergraph predictors in the suffix, $\mu = \maj\left(\mu_{S_{\le \lceil n /4 \rceil}}^{L}, \dots, \mu_{S_{\le n-1}}^{L}\right)$ satisfies
\begin{align*}
\cL_{\cD}\left(\mu\right) &= \EE_{(x,y) \sim D}\left[\mathds{1}\left[y\not\in\mu(x)\right]\right] \\
    &\le \EE_{(x,y) \sim D}\left[3\left(\frac{4}{3n}\sum_{t=\lceil n/4\rceil}^{n-1}\mathds{1}\left[y \not \in \mu_{S_{\le t}}^{L}(x) \right]\right)\right]\\
    &= 3\cdot\frac{4}{3n}\sum_{t=\lceil n/4\rceil}^{n-1}\cL_{\cD}\left(\mu_{S_{\le t}}^{L}\right)
    \\
    &\le 14.46\left(\frac{6\ddsh\log K }{n}+\frac{1}{n}\log\left(\frac{2}{\delta}\right)\right).
\end{align*}
In the first inequality above, we used the fact that when the majority vote is wrong, at least half of the predictors in the suffix are wrong, and the suffix contains \(n-\lceil n/4\rceil\ge n/2\) predictors.
This completes the proof.
\end{proof}

The proof is given by inductively proving the following lemma.
\begin{lemma}\label{claim: induc}
    With probability at least $1-(t\delta)/{\log K}$, the list predictor $\mu_t$ in \cref{eq: mu_t calculation} has error rate at most $(t\epsilon)/{\log K}$ for $t\in\{0, 1,\dots,\lfloor\log K\rfloor\}$.

\end{lemma}
\begin{proof}[Proof of \cref{claim: induc}]

For the base case, when $t = 0$, by definition we know that $\mu_{0}(x) = \cY$ for any $x\in \cX$, which satisfies our inductive assumption.

For the inductive step, we assume that $\cL_{\cD}\left(\mu_{t-1}\right)\le((t-1)\epsilon)/{\log K}$ holds with probability at least $1 - ((t-1)\delta)/\log K$. According to \cref{algo: Bandit PAC Multiclass Classification}, we predict the label $\hat{y}_i$ uniformly randomly from the list $\mu_{t-1}(x_i)$ for each instance $x_i$ we encountered in epoch $t$, we know that 
\[
    \Pr[\hat{y}_i = y_i] \ge \frac{1}{2L_{t-1}}\left(1- \frac{(t-1)\epsilon}{\log K}\right).
\]
Also, notice that, the number of samples we encountered in epoch $t$ is 
\[
m_t-m_{t-1}
:=
\left\lceil
\frac{4L_{t-1}}{1-\frac{(t-1)\epsilon}{\log K}}
\left(
15\frac{\left(6d_{DS}^{\lceil L_t/2\rceil}\log K+
\log\left(\frac{4\log K}{\delta}\right)\right)\log K}{\epsilon}
+
\log\left(\frac{2\log K}{\delta}\right)
\right)
\right\rceil.
\]

Conditioned on the history before epoch $t$, the indicators $\mathds{1}[\hat{y}_i = y_i]$ are independent Bernoulli random variables, a multiplicative lower-tail Chernoff bound yields that, with probabilty at least $1- \delta/(2\log K)$ we can collect a labeled sample $S_t$ whose size $n_t=|S_t|$ satisfies
\[
n_t\ge
15\frac{\left(6d_{DS}^{\lceil L_t/2\rceil}\log K+
\log\left(\frac{4\log K}{\delta}\right)\right)\log K}{\epsilon}.
\]

Since $\mu_{t-1}(x)$ is padded to have the same size for every $x$, conditional on $y\in\mu_{t-1}(x)$, every example is retained with the same probability. Hence, conditional on $\mu_{t-1}$ and on $n_t=|S_t|$, the retained examples in $S_t$ are i.i.d.\ from
\[
\cD_t:=\cD\mid\{(x,y):y\in\mu_{t-1}(x)\}.
\]

By \Cref{lemma: loo to pac}, with probability at least
$1-\delta/(2\log K)$,
\[
\cL_{\cD_t}(\mu_t)\le \frac{\epsilon}{\log K}.
\]
Moreover,
\begin{align*}
\cL_{\cD}(\mu_t)
&=
\Pr[y\notin\mu_t(x),\,y\in\mu_{t-1}(x)]
+
\Pr[y\notin\mu_t(x),\,y\notin\mu_{t-1}(x)]
\\
&\le
\Pr[y\in\mu_{t-1}(x)]\,
\cL_{\cD_t}(\mu_t)
+
\cL_{\cD}(\mu_{t-1})
\\
&\le
\cL_{\cD_t}(\mu_t)
+
\cL_{\cD}(\mu_{t-1})
\\
&\le
\frac{\epsilon}{\log K}
+
\frac{(t-1)\epsilon}{\log K}
=
\frac{t\epsilon}{\log K}.
\end{align*}

Condition on the event from the induction hypothesis, which holds with
probability at least \(1-(t-1)\delta/\log K\). The Chernoff step and the
list-learning PAC step each fail with probability at most
\(\delta/(2\log K)\). A union bound therefore gives the desired probability
\(1-t\delta/\log K\).
\end{proof}
Now, based on \cref{claim: induc}, we start to prove \cref{theorem: upper bound for bandit feedback}.
\begin{proof}[Proof of \cref{theorem: upper bound for bandit feedback}]
The case \(K=1\) is trivial, so throughout the proof assume \(K\ge2\). By \cref{claim: induc}, we know that with probability at least $1-\delta$ , the hypothesis $h_{S}$ returned by \cref{algo: Bandit PAC Multiclass Classification} has error at most $\epsilon$. Thus the sample complexity $\mathrm{\mathbf{m}}_{\cH}^{B}(\epsilon,\delta)$ can be upper bounded by the total number of samples used by \cref{algo: Bandit PAC Multiclass Classification}. Formally, we have
\begin{align*}
    \mathrm{\mathbf{m}}_{\cH}^{B}(\epsilon,\delta) &\le \sum_{t=1}^{\left\lfloor\log(K)\right\rfloor} \left(m_t - m_{t-1}\right)\\
    &\le O\left(\frac{\max_{1\le L\le K-1}\left(Ld_{DS}^{L}\right)(\log K)^3}{\epsilon} + \frac{K\log K\log\left(\frac{2\log K}{\delta}\right)}{\epsilon}\right).
\end{align*}
By the definition of the Bandit-DS dimension, for any concept class $\cH$,
   \[\operatorname{BDS}(\cH) \geq \max_{1\le L\le K-1}\left(L\cdot\mathrm{L}\text{-}\operatorname{DS}(\cH)\right).\] Combining the preceding inequalities proves the desired upper bound. The full
calculation is deferred to \cref{complete proof of upper bound}.
\end{proof}
\subsection{Proof of \cref{theorem: lower bound for bandit feedback}} 
In this section, we prove \cref{theorem: lower bound for bandit feedback} by constructing two hard cases.



\begin{proof}[proof of \cref{theorem: lower bound for bandit feedback}] We first show that we at least need $\Omega(\operatorname{BDS}(\cH)/\epsilon)$ samples. For a concept class $\cH\subseteq\cY^\cX$, let $\operatorname{BDS}(\cH) = d_{DS}^{B}< \infty$. By non-triviality, choose a BDS-shattered pair
$Z=(x_1,\ldots,x_d)\in\mathcal X^d$ and $N\in\mathbb N^d$
with $d\ge 2$ and
$\sum_{i=1}^d N_i=d_{\mathrm{DS}}^B$. Without loss of generality, we assume $N_i$'s are in ascending order. Since $N_1\le N_2$, we have
\[
\sum_{i=2}^{d}N_i\ge \frac{1}{2}d_{DS}^{B}.
\]


We consider the distribution $\cD$ over $\cX$ as follows
\begin{equation*}
    p(x) =\begin{cases}
			1 - 16\epsilon, & \text{if $x = x_1$}\\
            \frac{16 N_i\epsilon}{\sum_{i=2}^{d}N_i}, & \text{if $x \in \{ x_2,\cdots,x_d\} $}
		 \end{cases}
\end{equation*}
Let $\cF\subseteq\cH$, $|\cF|<\infty$, witness the BDS-shattering of $\cZ$, and let $V=\cF|_{\cZ}$. For each $v\in V$, fix one representative $c_v\in\cF$ with $c_v|_{\cZ}=v$. We choose $v^*$ uniformly at random from $V$ and set $c^*=c_{v^*}$.

For simplicity, let $\cZ^{\prime} = \{x_2,\cdots,x_{d}\}$. Let $m = (\sum_{i=2}^{d}N_i)/(64\epsilon)$ and $\mathcal{A}$ be any algorithm that uses an $\iid$ sample $S$ with size at most $m$ before picking a hypothesis. For ease of presentation, we define $\err^\prime(h) = \Pr[c^*(x)\neq h(x) \quad \text{and} \quad x \in \cZ^{\prime}].$ In expectation, for each point $x_i\in\cZ'$, there will be at most $N_i/4$ samples sampled from $x_i$. By Markov inequality, with probability at most $1/2$, there will be more than $N_i/2$ samples sampled from $x_i$. Let $A_i$ be the event that no more than $N_i/2$ samples are sampled from $x_i$. In other words, we have $\Pr_{S}[A_i] \ge 1/2$.
For the random target \(c^*=c_{v^*}\) chosen above, we have
\begin{align*}
    \mathbb{E}_{c^*, S}\left[\errp (h)\right] &= \sum_{i =2}^{d}\Pr[x=x_i]\Pr_{c^*,S}[h(x_i)\neq y_i]\\
    &\ge \sum_{i =2}^{d}\Pr[x=x_i]\Pr_{c^*, S}[h(x_i)\neq y_i|A_i] \Pr_{S}[A_i]\\
    &\ge \frac{1}{8}\sum_{i =2}^{d}\frac{16N_i\epsilon}{\sum_{i=2}^{d}N_i} = 2\epsilon.
\end{align*}

The first inequality is based on the law of total probability. For the second inequality, fix $i\ge2$ and condition on all coordinates of $c^*|_{\cZ}$ except the $i$th. Condition also on the learner's internal random seed and on the unlabeled
sample. The label $c^*(x_i)$ is then uniform among at least $N_i+1$ possible labels. On the event $A_i$, the learner has made at most $N_i/2$ guesses at $x_i$, so even including its final prediction, its conditional probability of error is at least
\[
1-\frac{N_i/2+1}{N_i+1}
=\frac{N_i}{2(N_i+1)}
\ge\frac14.
\]

This implies that there is some $c^*\in \cH$ such that $\mathbb{E}_{S}\left[\errp (h)\right] \ge 2\epsilon$. For the remainder of the proof, consider this $c^*$.
Note that, by definition, we have $\errp (h)\le \Pr[x\in\cZp] = 16\epsilon$. Let
\[
q=\Pr[\errp(h)>\epsilon],
\]
where the probability also includes the internal randomness of the learner.
Since $0\le\errp(h)\le16\epsilon$, we have
\[
2\epsilon
\le
\mathbb{E}_{S}[\errp(h)]
\le
16\epsilon q+\epsilon(1-q),
\]
and hence $q\ge1/15$. Consequently, for every algorithm using at most
\[
m=\frac{\sum_{i=2}^dN_i}{64\epsilon}
\]
samples, there exist a distribution $\cD$ and a target $c^*\in\cH$ such that
\[
\Pr[\err(h)>\epsilon]\ge\frac1{15}.
\]
Since
\[
\sum_{i=2}^dN_i\ge\frac12\operatorname{BDS}(\cH)
\]
and $\delta\le1/16<1/15$, this proves the
$\Omega(\operatorname{BDS}(\cH)/\epsilon)$ lower bound.

Then, we show that at least $\Omega(\log(1/\delta)/\epsilon)$ samples are necessary. Following the proof technique of \cite{blumer1989learnability}, we provide the argument here for completeness. Since $d\ge2$, by taking an $i$-neighbor pair from the BDS witness and any other coordinate, there exist two points $x_1,x_2$ and two hypotheses $h_1,h_2\in\cH$ such that $h_1,h_2$ agree on $x_1$ but disagree on $x_2$. We consider the distribution $\cD$ over $\cX$ as follows
\begin{equation*}
    p(x) =\begin{cases}
			1 - 2\epsilon, & \text{if $x = x_1$}\\
           2\epsilon, & \text{if $x = x_2 $}.
		 \end{cases}
\end{equation*}
Let $m$ be the sample size and $\mathcal{A}$ be any algorithm that uses an $\iid$ sample $S$ with size at most $m$ before picking a hypothesis. Let $B$ be the event that $m$ samples are not sampled from $x_2$. If the target is chosen uniformly from $\{h_1,h_2\}$, then conditioned on $B$, the learner receives the same transcript under the two targets. Therefore,
\begin{equation*}
    \frac{1}{2}\sum_{j=1}^{2}\Pr_{h_j}[\err(h)>\epsilon]
    \ge \frac{1}{2}\Pr_{S}[B]
    \geq \frac{1}{2}(1-2\epsilon)^m.
\end{equation*}
Hence, for at least one of the two targets,
\[
\Pr[\err(h)>\epsilon]\ge \frac{1}{2}(1-2\epsilon)^m.
\]
Thus, for $\epsilon\le 1/16$ and $\delta\le1/4$, achieving failure probability at most $\delta$ requires
\[
m=\Omega\left(\frac{\log(1/\delta)}{\epsilon}\right).
\]
Combining the two lower bounds proves the theorem.
\end{proof}

%% file: Main/Conclusion.tex
\section{Conclusion and Future Directions}

In this work, we introduce the bandit $\mathrm{DS}$ dimension, a new combinatorial complexity measure for multiclass PAC learning with bandit feedback. In terms of this dimension, we establish a near-optimal characterization of the realizable sample complexity, resolving, up to logarithmic factors, a longstanding open question posed by \citet{daniely2011multiclass}. Our upper bound is obtained via \textsc{ListCascade}, a new algorithmic framework that connects bandit learning with list learning by using list learners as intermediate predictors in a cascade, progressively reducing the ambiguity in the candidate labels until a single-label predictor is obtained. Along the way, we also prove a refined sample complexity bound for list PAC learning, which may be of independent interest. Two natural directions are particularly important: extending the approach to the agnostic setting, and determining whether the remaining logarithmic factors are necessary. In particular, it remains open whether every concept class $\mathcal{H}\subseteq\mathcal{Y}^{\mathcal{X}}$ with $\operatorname{BDS}(\mathcal{H})<\infty$ admits a bandit-feedback PAC learner with sample complexity
\(
    O\!\left(
    (\operatorname{BDS}(\mathcal{H})+\log(1/\delta))/\epsilon
    \right).
\)

%% file: Main/Acknowledgments.tex
\acksection

We thank the anonymous NeurIPS 2026 reviewers for their helpful feedback, which led to several minor changes in the paper.

Steve Hanneke acknowledges support by grant no.\ 2024243 from the United States - Israel Binational Science Foundation (BSF).
Shay Moran is a Robert J.\ Shillman Fellow; he acknowledges support by Israel PBC-VATAT, by the Technion Center for Machine Learning and Intelligent Systems (MLIS), and by the European Union (ERC, GENERALIZATION, 101039692). Views and opinions expressed are, however, those of the author(s) only and do not necessarily reflect those of the European Union or the European Research Council Executive Agency. Neither the European Union nor the granting authority can be held responsible for them.

%% file: Appendix/Definitions.tex
\section{Definitions}
In this section, we provide the auxilary formal definitions that will be used in our paper. We start from the notations.
\subsection{Notations} \label{Basic Notations}

In this subsection, we present the basic notation used in the paper; all of it is standard in the literature and is included for completeness. Let $\mathbb{N}$ and $\mathbb{R}$ stand for the set of natural numbers and real numbers, respectively. We denote by $\Bar{\mathbb{N}}$ the extended natural number system defined as $\Bar{\mathbb{N}} := \mathbb{N} \cup \{ +\infty \}$. In addition, for a given $n \in \mathbb{N}$, we use $[n]$ to denote $\big \{ 1, 2, \ldots, n \big \}$. Next, given $n \in \mathbb{N}$, for any sequence of size $n$ or $n$-tuple $v$, and any $i \in [n]$, let us use $v_i$ to denote the $i$-th element in $v$. Afterward, we denote by $A \times B$ the Cartesian product of two arbitrary sets $A$ and $B$. In addition, for any set $A$ and any $n \in \mathbb{N}$, we let $A^n$ indicate $n$ times the Cartesian product of $A$ with itself. Note that for any set $A$, we define $A^0 := \{()\},$. Also, given a set $A$, we denote by $A^{*}$ the set of all finite sequences of members of $A$; more formally, $A^{*} := \bigcup^{\infty}_{T = 0} A^T$. Then, for arbitrary sets $X$ and $Y$, we use $Y^X$ to denote the space of all functions from $X$ to $Y$. Additionally, let $\mathcal{F} \subseteq Y^X$ for the arbitrary sets $X$ and $Y$. Given a subset $S \subseteq X$, we define $\mathcal{F}|_{S}$ as follows: $\mathcal{F}|_{S} := \{ g \; | \; g \in Y^{S}, \: \exists_{f \in \mathcal{F}} \, \forall_{x \in S} \, g(x) = f(x) \}$. Similarly, give a sequence $S \in X^{*}$, we define $\mathcal{F}|_{S}$ as follows: $\mathcal{F}|_{S} := \{ g \; | \; g \in Y^{|S|}, \: \exists_{f \in \mathcal{F}} \, \forall_{i \in [|S|]} \, g_i = f(S_i) \}$. Going further, let $\mathrm{e}$ and $\log{(\cdot)}$ stand for Euler's number and the Logarithm function in the base $2$, respectively. Finally, we use $\mathcal{O}(\cdot)$, $\Omega(\cdot)$, and $\Theta(\cdot)$ to denote the standard asymptotic notation in theoretical computer science. Furthermore, we use $\widetilde{O}(\cdot)$, $\widetilde{\Omega}(\cdot)$, and $\widetilde{\Theta}(\cdot)$ to hide poly-logarithmic factors.

\subsection{Combinatorial Complexity Parameters} \label{Combinatorial Complexity Parameters}

In this subsection, we present the combinatorial complexity parameters considered in this work. For each parameter, we begin by defining the associated notion of a shattered sequence or shattered set, and then define the parameter itself in terms of that notion. Before proceeding, we state the following definition.

\begin{definition} [$i$-neighbor] \label{def:i-neighbor}
Let $\mathscr{Y}$ be a non-empty set. Let $f, g \in \mathscr{Y}^{d}$ for some $d \in \mathbb{N}$. For every $i \in [d]$, we say that $f$ and $g$ are $i$-neighbors if $f_i \neq g_i$ and $\forall_{j \in [d] \setminus \{i\}} \; f_j = g_j$.
\end{definition}

We adopt the convention that each dimension defined below is equal to $0$ when there is no positive-dimensional witness.

Next, we formally state the definition of the bandit $\operatorname{DS}$ dimension.

\begin{definition}[$\operatorname{BDS}$-shattered sequence] \label{def: bds-shattered}
Let $\mathcal{H} \subseteq \mathcal{Y}^\mathcal{X}$ be a concept class. Let $S \in \mathcal{X}^{d}$ be a sequence of instances for some $d \in \mathbb{N}$. Also, let $\mathrm{N} \in \mathbb{N}^{d}$. We say that $S$ and $\mathrm{N}$ is \emph{$\operatorname{BDS}$-shattered} by $\mathcal{H}$, if there exists $\emptyset \neq F \subseteq \mathcal{H}, |F| < \infty$ such that for all $f \in F|_{S}$ and for all $i \in [d]$, $f$ has at least $\mathrm{N}_i$ distinct $i$-neighbor in $F|_{S}$.
\end{definition}

\begin{definition}[Bandit $\operatorname{DS}$ dimension] \label{def: bds-dim}
Let $\mathcal{H} \subseteq \mathcal{Y}^\mathcal{X}$ be a concept class. The \emph{Bandit $\operatorname{DS}$ dimension} of $\mathcal{H}$, denoted by $\operatorname{BDS}(\mathcal{H}) \in \Bar{\mathbb{N}} \cup \{0\}$, is defined as the $\sup_{d \in \mathbb{N} \cup \{0\}}$ such that there exist a sequence of instances $S \in \mathcal{X}^{d^{\prime}}$ for some $d^{\prime} \in \mathbb{N}$ and $\mathrm{N} \in \mathbb{N}^{d^{\prime}}$ such that $\sum_{i = 1}^{d^{\prime}} N_i = d$ that is $\operatorname{BDS}$-shattered by $\mathcal{H}$.
\end{definition}

Next, we define combinatorial complexity parameters related to list PAC learning. Notably, when the list size is equal to one, each of these dimensions coincides with its corresponding counterpart in the multiclass setting. We begin with the definition of the $\operatorname{DS}_{\mathrm{L}}$ dimension.

\begin{definition}[$\operatorname{DS}_{\mathrm{L}}$-shattered sequence] \label{def: l-ds-shattered}
Let $\mathcal{H} \subseteq \mathcal{Y}^\mathcal{X}$ be a concept class, and $\mathrm{L} \in \mathbb{N}$. Let $S \in \mathcal{X}^{d}$ be a sequence of instances for some $d \in \mathbb{N}$. We say that $S$ is \emph{$\operatorname{DS}_{\mathrm{L}}$-shattered} by $\mathcal{H}$, if there exists $\emptyset \neq F \subseteq \mathcal{H}, |F| < \infty$ such that for all $f \in F|_{S}$ and for all $i \in [d]$, $f$ has at least $\mathrm{L}$ distinct $i$-neighbor in $F|_{S}$.
\end{definition}

\begin{definition}[$\operatorname{DS}_{\mathrm{L}}$ dimension \cite{charikar2023characterization}] \label{def: l-ds-dim}
Let $\mathcal{H} \subseteq \mathcal{Y}^\mathcal{X}$ be a concept class, and $\mathrm{L} \in \mathbb{N}$. The \emph{$\operatorname{DS}_{\mathrm{L}}$ dimension} of $\mathcal{H}$, denoted by $\operatorname{DS}_{\mathrm{L}}(\mathcal{H}) \in \Bar{\mathbb{N}} \cup \{0\}$, is defined as the $\sup_{d \in \mathbb{N} \cup \{0\}}$ such that there exists a sequence of instances $S \in \mathcal{X}^{d}$ that is $\operatorname{DS}_{\mathrm{L}}$-shattered by $\mathcal{H}$.
\end{definition}

\noindent Next, we introduce the $\mathrm{L}$-Exponential dimension. This combinatorial complexity parameter is particularly useful in shifting arguments.

\begin{definition}[$\operatorname{E}_{\mathrm{L}}$-shattered set] \label{def: l-e-shattered}
Let $\mathcal{H} \subseteq \mathcal{Y}^\mathcal{X}$ be a concept class, and $\mathrm{L} \in \mathbb{N}$. Let $S \subseteq \mathcal{X}$ be a set of instances of size $d$ for some $d \in \mathbb{N}$. We say that $S$ is \emph{$\operatorname{E}_{\mathrm{L}}$-shattered} by $\mathcal{H}$, if $\big | \mathcal{H}|_{S} \big | \geq (\mathrm{L} + 1)^{d}$.
\end{definition}

\begin{definition}[$\mathrm{L}$-Exponential dimension \cite{charikar2023characterization}] \label{def: l-exponential-dim}
Let $\mathcal{H} \subseteq \mathcal{Y}^\mathcal{X}$ be a concept class, and $\mathrm{L} \in \mathbb{N}$. The \emph{$\mathrm{L}$-Exponential dimension} of $\mathcal{H}$, denoted by $\operatorname{E}_{\mathrm{L}}(\mathcal{H}) \in \Bar{\mathbb{N}} \cup \{0\}$, is defined as the $\sup_{d \in \mathbb{N} \cup \{0\}}$ such that there exists a set of instances $S \subseteq \mathcal{X}$ of size $d$ that is $\operatorname{E}_{\mathrm{L}}$-shattered by $\mathcal{H}$.
\end{definition}

\subsection{Multiclass Learning with Bandit Feedback Algorithms} \label{appendix: Multiclass Learning with Bandit Feedback Algorithms}

To improve readability, we use a comma to separate indices when multiple indices are involved. For example, let $v$ be a sequence of length $5$ whose elements are $2$-tuples. We denote by $v_{5, 1}$ the first component of the $5$-th element of $v$. In addition, letting $m, n \in \mathbb{N}$ be such that $m \leq n$, we write $\big [ (x_i, y_i, z_i) \big ]_{i= m}^{n}$ to denote $\big ( (x_m, y_m, z_m), (x_{m + 1}, y_{m + 1}, z_{m + 1}), \dots, (x_n, y_n, z_n) \big )$. This section is similar in spirit to \cite{hanneke2025universal}.

Let $\Sigma:=\mathcal{Y}\times\{0,1\}$. Let $A=(A_1,A_2)$ be a multiclass PAC learning algorithm with
bandit feedback, where $A_1$ specifies the conditional distribution
of the next predicted label given the transcript and current
instance, and $A_2$ specifies the conditional distribution of
the final hypothesis given the transcript. For any $\mathfrak{u} \in \big ( \mathcal{X} \times \Sigma \big )^{*}$ and any $x \in \mathcal{X}$, let us write $\mathbf{A}_1(x; \mathfrak{u})$ to denote $\mathbf{A}_1 \big ( (\mathfrak{u}, x) \big )$. Also, for every finite sequence of examples $\mathcal{S} \in \mathcal{Z}^{n} = (\mathcal{X} \times \mathcal{Y})^{n}$ of size $n$ for some $n \in \mathbb{N}$, we denote by $\mathbf{B}_{\mathbf{A}}(\mathcal{S})$ the random variable taking values in $(\mathcal{X} \times \Sigma)^{n}$ representing the first part of the input of $\mathbf{A}_1$ in round $n + 1$ when the choices of nature are obtained according to $\mathcal{S}$, which is recursively defined as follows:
\begin{equation*}
    \mathbf{B}_{\mathbf{A}}(\mathcal{S})_{i, 1} := \mathcal{S}_{i, 1}, \quad 1 \leq i \leq n
\end{equation*}
\begin{equation*}
    \mathbf{B}_{\mathbf{A}}(\mathcal{S})_{i, 2} := 
    \begin{cases}
        y \sim \mathbf{A}_1 \big ( \mathcal{S}_{1, 1}; \emptyset \big ),
        & i = 1\\
        y \sim \mathbf{A}_1 \big ( \mathcal{S}_{i, 1}; \big [ \big ( \mathbf{B}_{\mathbf{A}}(\mathcal{S})_{j, 1}, \mathbf{B}_{\mathbf{A}}(\mathcal{S})_{j, 2}, \mathbf{B}_{\mathbf{A}}(\mathcal{S})_{j, 3} \big ) \big ]_{j=1}^{i - 1} \big ),
        & 1 < i \leq n
    \end{cases}
\end{equation*}
\begin{equation*}
    \mathbf{B}_{\mathbf{A}}(\mathcal{S})_{i, 3} := \mathds{1} \big \{ \mathbf{B}_{\mathbf{A}}(\mathcal{S})_{i, 2} = \mathcal{S}_{i, 2} \big \}, \quad 1 \leq i \leq n
\end{equation*}

Building upon that, for every finite sequence of examples
$S\in\mathcal Z^n=(\mathcal X\times\mathcal Y)^n$, we denote by
$\widehat h_A^S$ the random output hypothesis defined by
\[
    \widehat h_A^S \sim A_2(B_A(S)).
\]
Thus, $\widehat h_A^S$ is a single random element of
$\mathcal Y^{\mathcal X}$.

\subsection{One-inclusion Graphs} \label{OIG}

We now introduce an important combinatorial object associated with a hypothesis class, namely, the one-inclusion hypergraph.

\begin{definition}[One-inclusion hypergraph \cite{haussler1994predicting, rubinstein2006shifting}] \label{One-inclusion hypergraph}
The one-inclusion hypergraph of $\mathcal{H} \subseteq \mathcal{Y}^m$ for some $m \in \mathbb{N}$ is a hypergraph $G(\mathcal{H}) = (V, E)$ that is defined as follows. The vertex set is $V = \mathcal{H}$. For each $i \in [m]$ and $f : [m] \setminus \{i\} \to \mathcal{Y}$, let $e_{i, f}$ be the set of all $h \in \mathcal{H}$ that agree with $f$ on $[m] \setminus \{i\}$. The edge set is defined as follows:
\begin{equation*}
    E = \{e_{i, f} \; | \; i \in [m], f:[m] \setminus \{i\} \to \mathcal{Y}, e_{i,f} \neq \emptyset \}.
\end{equation*}
We regard the hyperedges as an indexed family: edges with distinct
indices $(i,f)$ remain distinct even when their vertex sets coincide.

Moreover, we say that the edge $e_{i,f} \in E$ is in the direction $i$, and is adjacent to the hypothesis/vertex $h$ if $h \in e_{i,f}$. In addition, every vertex $h \in V$ is adjacent to exactly $m$ edges. Further, the size of the edge $e_{i,f}$ is the size of the set $|e_{i,f}|$.
\end{definition}

For the one-inclusion hypergraph, we study the degrees of its vertices and its average degree.

\begin{definition}[$\mathrm{L}$-degree \cite{charikar2023characterization}] \label{def: L-degree}
Let $G(\mathcal{H}) = (V, E)$ be the one-inclusion hypergraph of $\mathcal{H} \subseteq \mathcal{Y}^m$ for some $m \in \mathbb{N}$. The $\mathrm{L}$-degree of a vertex $v \in V$ is defined as follows:
\begin{align*}
        \degree(v) &= |\{e \in E: v \in e, |e|>L\}|.
\end{align*}
\end{definition}

\begin{definition}[Average $L$-degree \cite{charikar2023characterization}] \label{def: avg-L-degree}
Let $G(\mathcal{H}) = (V, E)$ be the one-inclusion hypergraph of $\mathcal{H} \subseteq \mathcal{Y}^m$ for some $m \in \mathbb{N}$. The average $\mathrm{L}$-degree of $\mathcal{H}$ is defined as follows:
    \begin{align*}
        \avd(\mathcal{H}) := \frac{1}{|V|} \, \sum_{v \in V}\degree(v) = \frac{1}{|V|} \, \sum_{e \in E:|e|>L}|e|.
    \end{align*}
\end{definition}

Next, we restate the definition of a list orientation of the one-inclusion hypergraph from \cite{charikar2023characterization}. This notion captures the behavior of a deterministic learning algorithm that receives a labeled sample together with an unlabeled test point and outputs a prediction for that test point.

\begin{definition}[List Orientation \cite{charikar2023characterization}]
Let $G(\mathcal{H}) = (V, E)$ be the one-inclusion hypergraph of $\mathcal{H} \subseteq \mathcal{Y}^m$ for some $m \in \mathbb{N}$. A list orientation $\sigmal$ of $G(\mathcal{H}) = (V, E)$ having list size $\mathrm{L}$ is a mapping $\sigmal: E \to \{V'\subseteq V: |V'|\le L\}$ such that for each edge $e \in E$, $\sigmal(e) \subseteq e$.
\end{definition}

For a list orientation, we define its $\mathrm{L}$-outdegree.

\begin{definition}[$\mathrm{L}$-outdegree of a list orientation \cite{charikar2023characterization}]
Let $G(\mathcal{H}) = (V, E)$ be the one-inclusion hypergraph of $\mathcal{H} \subseteq \mathcal{Y}^m$ for some $m \in \mathbb{N}$, and let $\sigmal$ be a $\mathrm{L}$-list orientation of it. The $\mathrm{L}$-outdegree of $v \in V$ in $\sigmal$ is defined as follows:
\begin{equation*}
        \outdeg(v; \sigmal) := |\{e:v \in e, v \notin \sigmal(e)\}|.
\end{equation*}
The maximum $\mathrm{L}$-outdegree of $\sigmal$ is
\begin{equation*}
        \outdeg(\sigmal) := \sup_{v \in V} \outdeg(v;\sigmal).
\end{equation*}
\end{definition}

%% file: Appendix/Deferred.tex
\section{Deferred Proofs} \label{Deferred Proofs}

Next, we present \cref{algo: learn a L-list}, which is the subroutine of \cref{algo: Bandit PAC Multiclass Classification}, corresponding to the one-inclusion list learners in \cite{charikar2023characterization,hanneke2026optimalsauerlemmakary}. For the sake of completeness, we restate the algorithm below. 
\begin{algorithm}[H]
\caption{The one-inclusion list algorithm $\cA_{\cH}^{L}$ for $\cH\subseteq\cY^{\cX}$ of list size $L$.} 
\label{algo: learn a L-list}
\begin{flushleft}
  {\bf Input:} A sample $S = \big((x_1, y_1),\cdots,(x_m, y_m)\big)$ realizable by $\cH$ and list size $L$.\\
{\bf Output:} A $L$-list hypothesis $\cA^L_{\cH}(S)=\mu^L_S:\cX \to \{Y \subseteq \cY: |Y|\le L\}$.\\
\ \\
For each $x \in \cX$, the $L$-list $\mu^{L}_S(x)$ is computed as follows.
\end{flushleft}
\begin{algorithmic}[1]
\State Consider the class $\cH' \subseteq \cY^{m+1}$ of all patterns over the \emph{unlabeled data} that are realizable by $\cH$.
\State Find a $L$-list orientation $\sigma^{L}$ of $G(\cH')$ that \textit{minimizes} the \textit{maximum} $L$-outdegree.
\State Consider the edge in direction $m+1$ defined by $S$:
\[e =\{ h \in \cH':  \forall i \in [m]  \ \ 
h(i) =y_i\}.\]
\State Set $\mu^{L}_S(x) = \{h(m+1):h \in \sigma^{L}(e)\}$.
\end{algorithmic}
\end{algorithm}
In Algorithm 2, repeated training examples are deleted before
constructing the hypergraph, and $m$ denotes the remaining sample
size. The orientation in line 2 is the fixed orientation for
$U=\{x_1,\ldots,x_m,x\}$.

For every finite unlabeled set $U\subseteq\cX$, we fix once and for all
a minimizing $L$-list orientation of $G(\cH|_U)$. The chosen orientation
depends only on $U$, and not on the order in which its elements are
presented or on the distinguished query point. If the query point already
appears in the training sample, the learner outputs an arbitrary $L$-list
containing its observed label.

Next we show the complete calculation in the proof of \cref{lemma: bound exp by L/2 DS} and \cref{theorem: upper bound for bandit feedback}
\subsection{Complete calculation in the proof of \cref{lemma: bound exp by L/2 DS}}\label{complete proof of dexp}
\begin{proof}[Proof of \cref{lemma: bound exp by L/2 DS}]
    If $d_{DS}^{\lceil L/2\rceil}=0$ or $d_E^L=0$, then the claim is immediate. If $d_E^L\le d_{DS}^{\lceil L/2\rceil}$, then the claim is also immediate. Hence, assume
    \[
    d_E^L>d_{DS}^{\lceil L/2\rceil}\ge1.
    \]
    Consider the set $S$ which realizes $\dexp$, so that $|S| = \dexp$. Then we have that
    \[|\cH|_{S}|\ge(L+1)^{\dexp}.\]
    On the other hand, by Theorem 2.1 of \cite{hanneke2026optimalsauerlemmakary}, we know that
    \[|\cH|_{S}|\le \left\lceil\frac{L}{2}\right\rceil^{\dexp - \ddsh}\left(\frac{eK\dexp}{\ddsh}\right)^{\ddsh}.\]
    Combining the above two inequalities, we have 
    \[(L+1)^{\dexp}\le \left\lceil\frac{L}{2}\right\rceil^{\dexp - \ddsh}\left(\frac{eK\dexp}{\ddsh}\right)^{\ddsh}.\]
    Note that $\lceil L/2\rceil\le \frac{L+1}{2}$.
    \[(L+1)^{\dexp}\le \left(\frac{L+1}{2}\right)^{\dexp - \ddsh}\left(\frac{eK\dexp}{\ddsh}\right)^{\ddsh}.\]
    By rearranging the term, we have
    \[\left(\frac{2L+2}{L+1} \right)^{\dexp} \le \left(\frac{2eK\dexp}{(L+1)\ddsh}\right)^{\ddsh} .\]
    Let $x=\dexp/\ddsh$. Rearranging the preceding inequality gives
    \[
    2^x\le \frac{2eKx}{L+1}\le eKx,
    \]
    where we used $L\ge1$. Equivalently,
    \[
    x\le \log(eK)+\log x.
    \]
    If $x\le4$, then $x\le 4\log K\le6\log K$. Otherwise,
    $\log x\le x/2$, and hence
    \[
    x\le 2\log(eK)\le6\log K,
    \]
    where the last inequality uses $K\ge2$. Therefore
    \[
    \dexp\le6\ddsh\log K.
    \]
\end{proof}

\subsection{Complete calculation in the proof of \cref{theorem: upper bound for bandit feedback}}
\label{complete proof of upper bound}
\begin{proof}[Proof of \cref{theorem: upper bound for bandit feedback}]
Based on \cref{claim: induc}, we know that with probability at least $1-\delta$ , the hypothesis $h_{S}$ returned by \cref{algo: Bandit PAC Multiclass Classification} has error rate at most $\epsilon$. Thus the sample complexity $\mathrm{\mathbf{m}}_{\cH}^{B}(\epsilon,\delta)$ can be upper bounded by the sample size of \cref{algo: Bandit PAC Multiclass Classification}. Formally, we have
\begin{align*}
    \mathrm{\mathbf{m}}_{\cH}^{B}(\epsilon,\delta)
    &\le \sum_{t=1}^{\left\lfloor\log K\right\rfloor}(m_t-m_{t-1})\\
    &\le
    \sum_{t=1}^{\left\lfloor\log K\right\rfloor}
    \frac{4L_{t-1}}{1-\frac{(t-1)\epsilon}{\log K}}
    \left(
    15\frac{
    \left(
    6d_{DS}^{\lceil L_t/2\rceil}\log K
    +
    \log\left(\frac{4\log K}{\delta}\right)
    \right)\log K
    }{\epsilon}
    +
    \log\left(\frac{2\log K}{\delta}\right)
    \right)
    +
    \left\lfloor\log K\right\rfloor\\
    & \le \frac{360\max_{1\le t\le\left\lfloor\log(K)\right\rfloor}
    \left(L_{t-1}d_{DS}^{\lceil L_t/2\rceil}\right)(\log K)^3}
    {\epsilon(1-\epsilon)}
    \\
    &\qquad+
    \frac{120K\log K\log\left(\frac{4\log K}{\delta}\right)}
    {\epsilon(1-\epsilon)}
    +
    \frac{8K\log\left(\frac{4\log K}{\delta}\right)}
    {1-\epsilon}.
\end{align*}
The additional $\lfloor\log K\rfloor$ term comes from rounding each epoch
length upward and is absorbed by the final asymptotic bound.
When $\epsilon\in(0,0.1]$, we have
\begin{align*}
\mathrm{\mathbf{m}}_{\cH}^{B}(\epsilon,\delta)
&\le
400\frac{\max_{1\le t\le\left\lfloor\log(K)\right\rfloor}
\left(L_{t-1}d_{DS}^{\lceil L_t/2\rceil}\right)(\log K)^3}{\epsilon}
\\
&\qquad+
150\frac{K\log K\log\left(\frac{4\log K}{\delta}\right)}{\epsilon}.
\end{align*}

Since
\[
L_{t-1}\le 2L_t\le 4\left\lceil\frac{L_t}{2}\right\rceil,
\]
we have
\[
L_{t-1}d_{DS}^{\lceil L_t/2\rceil}
\le
4\max_{1\le L\le K-1}\left(Ld_{DS}^{L}\right).
\]
Therefore,
\begin{align*}
\mathrm{\mathbf{m}}_{\cH}^{B}(\epsilon,\delta)
&\le
1600\frac{\max_{1\le L\le K-1}\left(Ld_{DS}^{L}\right)(\log K)^3}{\epsilon}
\\
&\qquad+
150\frac{K\log K\log\left(\frac{4\log K}{\delta}\right)}{\epsilon},
\end{align*}
In addition, for $\epsilon\ge 0.1$, run the algorithm with accuracy parameter
$\epsilon'=\min\{\epsilon,0.1\}$; since
$1/\epsilon'\le 10/\epsilon$, the same asymptotic bound holds for every
$\epsilon\in(0,1)$. This finishes the proof.
    


\end{proof}

%% file: Appendix/effective.tex
\section{Effective Size of the Label Space} \label{sec:effective-label-space}

The dependence on $|\mathcal{Y}|$ in the current statement is an artifact of using $|\mathcal{Y}|$ in the analysis. Moreover, our results extend to an infinite label space after replacing the cardinality of the full label space with the effective label space size. In particualr, the relevant quantity is the maximum number of labels realizable at any single instance:
$$
K_{\mathrm{eff}}(\mathcal{H}) = \sup_{x \in \mathcal{X}} \; \left| \{ h(x) \; | \; h \in \mathcal{H} \} \right|.
$$

To begin, when $K_{\mathrm{eff}}(\mathcal{H}) < \infty$, the class can be transformed into a class over the finite label space
$$
[K_{\mathrm{eff}}(\mathcal{H})].
$$
For each $x \in \mathcal{X}$, let
$$
\mathcal{H}(x) := \{h(x) \; | \; h \in \mathcal{H}\},
$$
choose a set $\mathcal{Y}_x\subseteq\mathcal{Y}$ of size
$K_{\mathrm{eff}}(\mathcal{H})$ containing $\mathcal{H}(x)$, and fix a bijection
$$
\phi_x : \mathcal{Y}_x \longrightarrow [K_{\mathrm{eff}}(\mathcal{H})].
$$
Define the transformed class
$$
\widetilde{\mathcal{H}} := \left\{ \widetilde{h} \in [K_{\mathrm{eff}}(\mathcal{H})]^{\mathcal{X}} \; | \; \exists \, h \in \mathcal{H} \; \text{such that} \; \forall \, x \in \mathcal{X} \widetilde{h}(x) = \phi_x(h(x)) \right\}.
$$
Since each $\phi_x$ is injective, and every prediction in
$[K_{\mathrm{eff}}(\mathcal{H})]$ can be mapped back through $\phi_x^{-1}$, this transformation preserves all coordinate-wise equality and inequality relations. Consequently, restriction of the class to any finite sequence of instances preserves neighbor relations, pseudo-boxes, the relevant combinatorial dimensions, and the associated one-inclusion hypergraphs. Having bandit feedback does not pose any obstacle. Therefore, ListCascade and its analysis would yield an upper bound with $K_{\mathrm{eff}}(\mathcal{H})$ in place of $|\mathcal{Y}|$.

We next explain why an infinite effective label space precludes uniform bandit PAC learnability. Indeed, fix any sample size $m \in \mathbb{N}$, and choose a point $x \in \mathcal{X}$ at which more than $2m + 1$ distinct labels are realizable by $\mathcal{H}$. Note that 
such a point exists by the definition of $K_{\mathrm{eff}}(\mathcal{H})$. Fix $2m+2$ distinct realizable labels $y_1,\ldots,y_{2m+2}$ at this point. For each realizable label $y$ at $x$, let $\mathcal{D}_y$ be the distribution concentrated on the single labeled example $(x, y)$. We then choose one of the distributions
$$
\{\mathcal{D}_{y_1},\ldots,\mathcal{D}_{y_{2m+2}}\}
$$
uniformly at random. In $m$ bandit interactions, the learner can test at most $m$ labels, and after receiving only negative feedback, it can correctly identify at most one additional untested label through its final output. Hence, its average probability of identifying the target label is at most
$$
\frac{m+1}{2m+2} \leq \frac{1}{2}.
$$
Therefore, for at least one realizable target, the learner outputs an incorrect label at $x$ with probability at least $1/2$. Since the distribution is concentrated on $x$, the resulting classification error is one. Thus, for any $\epsilon<1$ and $\delta<1/2$, no finite sample size can provide a uniform PAC guarantee.

Now, observe that having a finite $\operatorname{BDS}$ dimension rules out this obstruction. In fact, we have:
$$
K_{\mathrm{eff}} - 1 \leq \operatorname{BDS}(\mathcal{H}) \leq (K_{\mathrm{eff}}-1) \, \operatorname{DS}(\mathcal{H}).
$$
For the first inequality, $r$ realizable labels at one point give a one-coordinate $\operatorname{BDS}$ witness with $r-1$ neighbors. For the second inequality, every coordinate of a $\operatorname{BDS}$ witness contributes at most $K_{\mathrm{eff}}(\mathcal{H})-1$ neighbors, while the number of coordinates is at most $\operatorname{DS}(\mathcal{H})$.

As a result, the suggested possibility of a class with $\operatorname{BDS}(\mathcal{H}) = 1$ but infinite bandit PAC sample complexity cannot occur.

If $K_{\mathrm{eff}}(\mathcal H)=1$, all hypotheses agree
pointwise and zero samples suffice. The following bound
applies when $K_{\mathrm{eff}}(\mathcal H)\ge2$. Finally, after replacing $|\mathcal{Y}|$ by $K_{\mathrm{eff}}(\mathcal{H})$, our upper bound becomes
$$
O \left( \frac{\operatorname{BDS}(\mathcal{H}) \bigl(\log K_{\mathrm{eff}}(\mathcal{H})\bigr)^3 + K_{\mathrm{eff}}(\mathcal{H}) \log K_{\mathrm{eff}}(\mathcal{H}) \log\left(\frac{2\log K_{\mathrm{eff}}(\mathcal{H})}{\delta}\right)}{\epsilon} \right).
$$
Since
$$
K_{\mathrm{eff}}(\mathcal{H}) \leq \operatorname{BDS}(\mathcal{H}) + 1,
$$
the apparent dependence on the effective label-space size is absorbed by the $\operatorname{BDS}$ dimension up to logarithmic factors.

%% file: Appendix/Lemmata.tex
\section{Concentration Inequalities} \label{Concentration Inequalities}

This section collects standard concentration inequalities for sums of independent random variables, which are used throughout our analysis for the reader's convenience.

\begin{lemma}[Chernoff Bound]
Let $(X_1, X_2, \ldots, X_n)$ be a sequence of $n$ independent Bernoulli random variables for some $n \in \mathbb{N}$, and define $X := \sum_{i=1}^{n} X_i$. Moreover, let $\mu = \mathbb{E}[X]$. Then, the following inequalities hold.
\begin{enumerate}
    \item \textbf{(Upper Tail)} For any $\delta > 0$, we have:
    $$
    \mathbb{P} \big[ X \geq (1 + \delta) \mu \big] \leq \exp \left( -\dfrac{\delta^2 \mu}{2 +\delta} \right).
    $$
    
    \item \textbf{(Lower Tail)} For any $0 < \delta < 1$, we have:
    $$
    \mathbb{P} \big[ X \leq (1 - \delta) \mu \big] \leq \exp \left( -\dfrac{\delta^2 \mu}{2} \right).
    $$
\end{enumerate}
\end{lemma}

\begin{lemma}[Bernstein's Inequality]
Let $(X_1, X_2, \ldots, X_n)$ be a sequence of $n$ independent real-valued random variables with zero mean (i.e., for every $i \in [n]$, we have: $\mathbb{E}[X_i] = 0$). Suppose that there exists a constant $M \in  \mathbb{R}^{+}$ such that for every $i \in [n]$, we have: $|X_i| \le M$ almost surely. Moreover, define $S_n := \sum_{i=1}^n X_i$, and $\sigma^2 := \sum_{i=1}^n \mathbb{E}[X_i^2]$. Then for any $t > 0$, the following holds.
$$
\mathbb{P} \big[ S_n \geq t \big] \leq \exp \left( -\dfrac{t^2}{2\sigma^2 + \frac{2}{3} M t} \right).
$$
Furthermore, applying the same bound to both $S_n$ and $-S_n$ yields the two-sided inequality below.
$$
\mathbb{P} \big[ |S_n| \geq t \big] \leq 2 \exp \left( -\dfrac{t^2}{2\sigma^2 + \frac{2}{3}Mt} \right).
$$
\end{lemma}